\documentclass[conference,12pt,onecolumn]{ieeeconf} 

\IEEEoverridecommandlockouts                              

\title{\LARGE \bf
$n$-Step Temporal Difference Learning with Optimal $n$ 
}

\author{Lakshmi Mandal$^{\dagger}$ and Shalabh Bhatnagar$^{\dagger}$
\thanks{This work was supported by a J.C. Bose Fellowship, Project No.~DFTM/02/3125/M/04/AIR-04 from DRDO under DIA-RCOE, the Walmart Centre for Tech Excellence, IISc, and the RBCCPS, IISc.}
\thanks{$^{\dagger}$ The authors are with the Department of Computer Science and Automation,
        Indian Institute of Science, Bangalore 560012, India. E-mail:
        {\tt\small $\{$lmandal, shalabh$\}$@iisc.ac.in}}%
}

\usepackage{lipsum}
\usepackage{setspace}
\usepackage{balance} 
\usepackage[utf8]{inputenc} 
\usepackage{booktabs}       
\usepackage{amsfonts}       
\usepackage{amsmath}
\usepackage{lipsum}
\usepackage{graphicx}
\usepackage[ruled,vlined]{algorithm2e}
\usepackage{caption}
\usepackage{subcaption}	
\graphicspath{ {./images/} }
\usepackage{colortbl}
\usepackage{multirow}
\usepackage{multicol}
\usepackage{mathrsfs}
\usepackage{cleveref} 
\usepackage{url}
\usepackage[english]{babel}
\newtheorem{theorem}{Theorem}
\newtheorem{lemma}[theorem]{Lemma}
\newtheorem{proposition}{Proposition}
\newtheorem{definition}{Definition}
\newtheorem{remark}{Remark}

\begin{document}

\maketitle
\thispagestyle{empty}
\pagestyle{empty}

\begin{abstract}
We consider the problem of finding the optimal value of $n$ in the $n$-step temporal difference (TD) learning algorithm. 
Our objective function for the optimization problem is the average root mean squared error (RMSE).
We find the optimal $n$ by resorting to a  model-free optimization technique involving a one-simulation simultaneous perturbation stochastic approximation (SPSA) based procedure. Whereas SPSA is a zeroth-order continuous optimization procedure, we adapt it to the discrete optimization setting by using a random projection operator. 
We prove the asymptotic convergence of the recursion by showing that the sequence of $n$-updates obtained using zeroth-order stochastic gradient search converges almost surely to an internally chain transitive invariant set of an associated differential inclusion. This results in convergence of the discrete parameter sequence to the optimal $n$ 
in $n$-step TD. Through experiments, we show that the optimal value of $n$ is achieved with our SDPSA algorithm for arbitrary initial values. We further show using numerical evaluations that SDPSA outperforms the state-of-the-art discrete parameter stochastic optimization algorithm `Optimal Computing Budget Allocation (OCBA)' on benchmark RL tasks.
\end{abstract}

\section{Introduction}
Reinforcement learning (RL) algorithms are widely used for solving problems of sequential decision-making under uncertainty. An RL agent typically makes decisions based on data that it collects through interactions with the environment in order to maximize a certain long-term reward objective \cite{Bertsekas_1995,RL_Book, bertsekas2019reinforcement, meyn2022control}. Because of their model-free nature, RL algorithms have found extensive applications in various areas such as operations research, game theory, multi-agent systems, autonomous systems, communication networks and adaptive signal processing. Various classes of procedures such as the action-value methods, evolutionary algorithms, and policy gradient approaches are available for finding solutions to RL problems \cite{2012_Marco}. A widely popular class of approaches is the action-value methods that solve an RL problem by learning the action-value function under a given policy that is then used to design a better policy. 

Monte Carlo (MC), and temporal-difference (TD) learning are two popular model-free action-value methods \cite{RL_Book} that have been explored for their theoretical guarantees as well as in applications.
For example, the MC gradient method utilizing the sample paths is used to compute the policy gradient in \cite{Zhang2020}. Furthermore, the MC method is applied for RL-based optimal control of batch processes in \cite{Yoo2021}. A finite time analysis of a decentralized version of the TD(0) learning algorithm is discussed in \cite{Sun2020}. 
A variance-reduced TD method is investigated in \cite{Xu2020} that is seen to converge to a neighborhood of the fixed-point solution of TD at a linear rate. More recently, MC and TD algorithms based on function approximation using neural network architectures are being used as building blocks for deep reinforcement learning (Deep RL) \cite{Hoel2020}. In terms of adaptive algorithms, \cite{Shokri2019} presents an adaptive fuzzy approach combining eligibility traces with off-policy methods with decisions on whether or not to apply eligibility traces made during the exploration phase. This is however significantly different from the problem that we consider. 

It is to be noted that, in most cases, there is no proper justification for why only regular TD, also called 1-step TD, is used and not its $n$-step  generalization for $n>1$. The $n$-step TD (for $n> 1$) tends to fall in between MC methods and 1-step TD (also called TD(0)). This is because MC methods are full-trajectory approaches that perform value function updates only when the entire trajectory becomes available to the decision maker. The 
$n$-step TD methods combine advantages of both MC and TD requiring $n$ state transitions before beginning to perform updates. It is empirically observed in Chapter 7 of \cite{RL_Book} that different values of $n$ in general result in different performance behaviors in terms of the root mean squared error (RMSE). Thus, a natural question that arises is what value of $n$ should one use in a given problem setting and for a given choice of the step-size parameters. However, finding the optimal $n$ for various problem settings is challenging and a priori an optimal $n$ for any setting cannot be found directly. We propose in this paper, a two-step systematic procedure based on zeroth-order stochastic gradient search (though in the space of discrete parameters) that is purely data-driven and provides the optimal value of $n$ in $n$-step TD.

Methods for solving nonlinear (continuous) optimization problems include the Newton-Raphson technique, steepest descent, and quasi-Newton algorithms. 
Most stochastic optimization algorithms are stochastic variants of their aforementioned deterministic counterparts  \cite{bhatnagar-book, Maryak2001}. 
In cases where the objective function gradient is unknown and only noisy objective function estimates are available (a setting previously considered by Robbins and Monro \cite{rm} in a general context), Kiefer and Wolfowitz \cite{kw} presented a stochastic optimization procedure based on finite-difference gradient approximation that requires two function measurements or simulations for a scalar parameter. 
In the case of vector parameters, random search procedures such as simultaneous perturbation stochastic approximation (SPSA) \cite{Spall1992}, \cite{Spallt1997} and smoothed functional (SF) algorithms \cite{bhatnagar2007} have been found to be effective as the algorithms here only require one or two system simulations regardless of the parameter dimension.

Amongst discrete optimization methods for moderate-sized problems, ranking and selection approaches are found to be efficient \cite{bechhofer}. The optimal computing budget allocation (OCBA) is one such procedure that is based on Bayesian optimization \cite{chenlee} and is considered to be the best amongst ranking and selection procedures.
In \cite{Bhatnagar2011}, a stochastic discrete optimization problem is considered and two-simulation SPSA as well as SF algorithms, originally devised for problems of continuous (stochastic) optimization have been adapted to the case of discrete (stochastic) optimization. 

We present in this paper an adaptive zeroth order SPSA-type algorithm except that the perturbations are deterministic, cyclic and $\pm 1$-valued. This requires only one simulation at any parameter update unlike regular SPSA that requires two simulations. 
A general treatment on deterministic perturbation SPSA is provided in \cite{bhatnagar2003}. 
Our algorithm is a two-timescale stochastic approximation procedure where on the slower timescale we update the value of $n$ and on the faster scale, the {long-run average variance of the $n$-step TD estimator is obtained. In addition, the value function updates using $n$-step TD are performed on the faster timescale as well, that in turn feed into the long-run variance update step.} We prove the convergence of this scheme to the optimal $n$ that minimizes the long-run variance in $n$-step TD.

A significant challenge with the convergence analysis of the proposed scheme is that the underlying ordinary differential equation (ODE) is found to have a discontinuous RHS making the standard ODE approach to stochastic approximation inapplicable. This is because for an ODE based analysis to work, one requires the RHS of the ODE to be Lipschitz continuous. We work around this problem by suitably identifying a differential inclusion (DI) instead and showing that the stochastic recursion asymptotically tracks such a DI.
We finally show the results of experiments on two different RL environments, namely the Random Walk (RW) and Grid World (GW), where we observe that our scheme converges in each case to the optimal value of $n$ regardless of the initial choice of the same. The contributions of our work are summarized as follows.
 \begin{enumerate}
     \item We propose a Simultaneous Deterministic Perturbation Stochastic Approximation (SDPSA) algorithm that incorporates a one-simulation discrete optimization procedure but with deterministic perturbation sequences (for improved performance). This is unlike regular SPSA that works on continuous settings and with  random perturbations.
     \item The proposed SDPSA algorithm is a two-timescale procedure that we employ to find the optimal value of $n$ in $n$-step TD, starting from any arbitrary initial value of the same. We show that the gradient estimation procedure results in regular bias cancellation from the use of deterministic perturbations. 
     \item 
     We prove the asymptotic convergence of the SDPSA algorithm
     by showing that the recursion tracks the limit points of an associated differential inclusion (that we identify precisely), which we show demonstrates that SDPSA finds the optimal $n$ that minimizes the long-run variance in $n$-step TD.
     \item Our experimental results conducted on two different environments uniformly show that the optimal value of $n$ is achieved with SDPSA for any arbitrary initial value of the same\footnote{{The code for our experiments is available at  \url{https://github.com/LakshmiMandal/n_stepTD_opt_n.}}}. 
     \item We also show numerical comparisons of our algorithm SDPSA with the discrete parameter stochastic optimization algorithm OCBA. The latter (as mentioned previously) is widely recognized as the best ranking and selection procedure that makes efficient use of the available simulation budget. Our algorithm shows better results than OCBA in terms of both RMSE as well as computational time performance.
 \end{enumerate} 
Finally, we mention that in recent times, there is a lot of research activity on finite-time (non-asymptotic) analyses of reinforcement learning algorithms, see for instance, \cite{galdalal, wuetal}, that are based on obtaining probabilistic concentration bounds on the mean-squared error. Such analyses however rely on strong regularity conditions on the driving vector field and in particular, one requires the same to be Lipschitz continuous, see \cite{galdalal, wuetal}. As mentioned before, such conditions are clearly violated in our setting since our objective function derivative is in fact discontinuous. Hence, we treat the algorithm as a stochastic recursive inclusion whose analysis we carry out by forming an underlying DI. To the best of our knowledge, there is no work in the literature that deals with finite time bounds for such algorithms. In fact, as already mentioned, even performing an ODE-based asymptotic analysis is impossible precisely for the aforementioned lack of regularity conditions. Thus, an important feature of our work is that we identify a suitable DI for our system and then prove the asymptotic convergence of our algorithm to the limit points of this DI (without requiring the aforementioned regularity conditions).

\section{The Proposed Methodology}
\label{methodology}

This section describes our proposed SDPSA algorithm (Algorithm \ref{SPSA}) for determining the optimal $n$ for the $n$-step TD (TD($n$)) algorithm (Algorithm \ref{n_step_TD}) (see Figure  \ref{fig:flowchart}).

\begin{figure}[htb]
    \centering
    \vspace{-8pt}
    \includegraphics[width=3.5in]{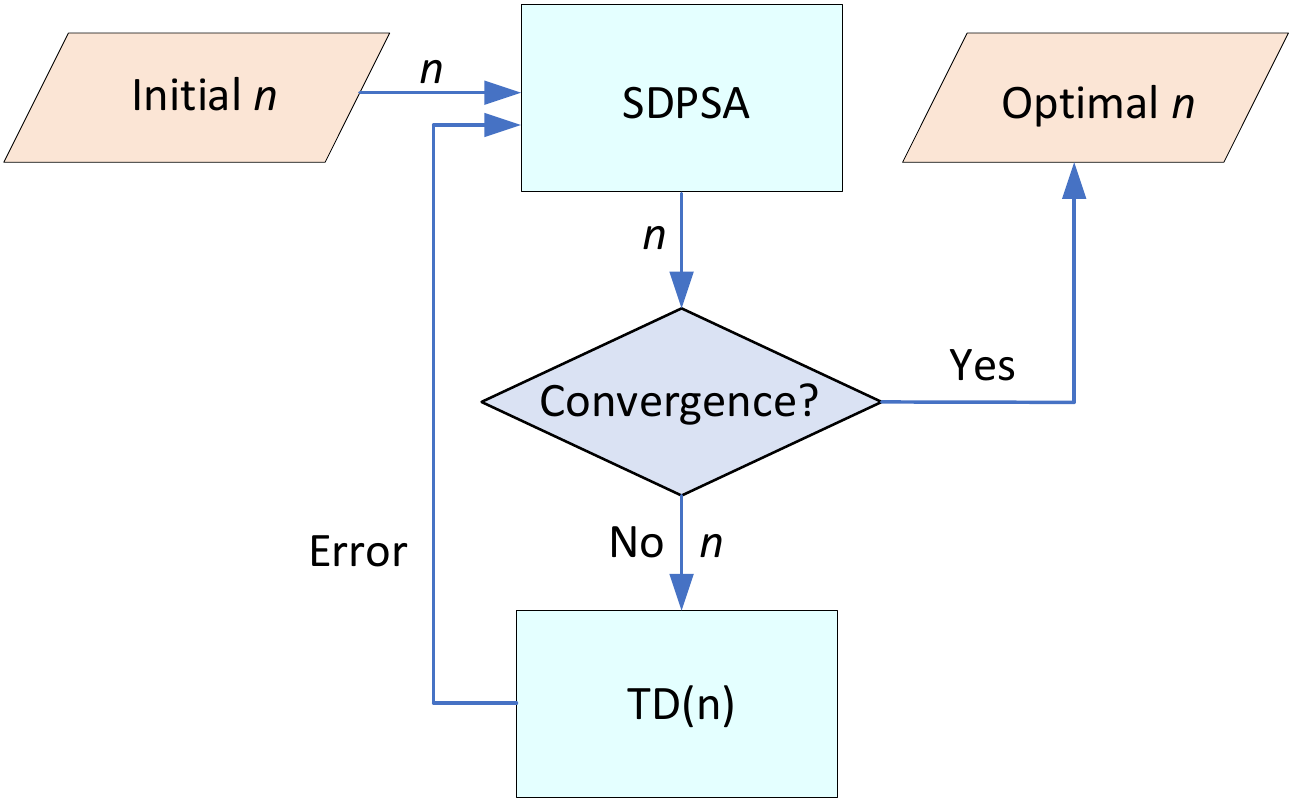}
    \caption{Flowchart of the proposed algorithm.}
    \label{fig:flowchart}
\end{figure}

\subsection{$n$-step TD (TD($n$)) Algorithm}
\begin{algorithm}[htb]
\caption{TD($n$) Algorithm} \label{n_step_TD}
\textbf{Input:} policy $\pi$, $n$, $\alpha$, initial  $V_n(s), \forall s \in S$, $J(n)=0$.\\
\textbf{Output:} Converged value of $V_n(s)$ $ \forall s \in S$ and $J(n)$.
For {$e$ in episodes:}{

~~~~Initialize: $S_0$ and set as non-terminal state
    
~~~~Termination time, $T \leftarrow \infty$

~~~~For {$i=0,1,2,\dots$}{

~~~~~~~~If {$i < T :$}{

~~~~~~~~~~~~Action $\sim \pi(\cdot \mid S_i)$ 
            
~~~~~~~~~~~~{Observe reward $R_{i+1}$, next state $S_{i+1}$} 

~~~~~~~~~~~~If {$S_{i+1}$ is terminal state:}{

~~~~~~~~~~~~~~~~$T \leftarrow i+1$}
~~~~~~~~}

~~~~~~~~$\kappa \leftarrow i-n+1$

~~~~~~~~If {$\kappa \geq 0:$}{

~~~~~~~~~~~~${\hat V_n}(S_\kappa)\leftarrow \sum_{j=\kappa+1}^{\min (\kappa +n ,T)} \gamma^{j-\kappa-1}R_j $

~~~~~~~~~~~~If {$\kappa +n < T:$}{

~~~~~~~~~~~~~~~~${\hat V_n}(S_\kappa) \leftarrow {\hat V_n}(S_\kappa) + \gamma^n V_n(S_{\kappa + n})$}

~~~~~~~~~~~~$V_n(S_{\kappa}) \leftarrow V_n(S_{\kappa})+ \alpha[{\hat V_n}(S_\kappa)-V_n(S_{\kappa})]$

~~~~~~~~~~~~{$g_n(S_{\kappa})\leftarrow (\hat V_n(S_{\kappa})-V_n(S_{\kappa}))^2$}
            
~~~~~~~~~~~~$J(n) \leftarrow J(n) + \alpha[g_n(S_\kappa) - J(n)]$
            }
        }

~~~~Until $\kappa = T-1$}

\end{algorithm}

Our proposed algorithm is a two-timescale stochastic approximation scheme wherein the TD($n$) recursion runs along the faster timescale for any prescribed value of $n$ that in-turn gets updated on the slower timescale along a gradient descent direction. The basic underlying setting is of a Markov decision process (MDP), which is defined via the tuple $<S; A; P; R; \gamma>$, where 
$S$ denotes the state space, $A$ is the action space, $P$ is the transition probability function, i.e., $P:S \times S \times A \rightarrow [0,1]$, $R$ denotes the reward function, and $\gamma \in (0,1)$ denotes the discount factor, respectively. The state and action spaces are assumed here to be finite.

We now describe the TD($n$) algorithm (see Algorithm \ref{n_step_TD}). 
Here, a policy $\pi$, which maps the state space $S$ to the action space $A$, is input to the algorithm. The {input} parameters of Algorithm \ref{n_step_TD} are a positive integer $n$ and a step-size $\alpha \in (0,1]$. 
The value function {$V_n(\cdot)$, for a given $n$}, for each state $s \in S$ is initialized arbitrarily. {Here $R_t\stackrel{\triangle}{=} R(S_{t-1},\pi(S_{t-1}), S_t)$ is the reward obtained at instant $t$}. 
The $n$-step return of TD($n$) for the state visited at instant $\kappa$ in the trajectory is denoted by ${\hat V_n}(S_\kappa)$, where for $\kappa=i$,
\begin{equation}
\label{hatV}
{\hat V_n}(S_i) \doteq \sum_{j=1}^n \gamma ^{j-1}R_{i+j} +\gamma^n V_n(S_{i+n}),
\end{equation}
$\forall n\geq 1$ and $0 \leq i< T-n$. In this algorithm, ${\hat V_n}(S_i)$ is the $n$-step return used in TD($n$) counted from instant $i$ when the state is $S_i$. The quantity $V_{n}(S_{i+n})$ denotes the estimate of the value of state $S_{i+n}$. 
Note that the expected value of $\hat{V}_n(S_i)$ given state $S_i$ is the true value of state $S_i$ under the given policy. 
The error term $g_n(\cdot)$ and its long-term average $J(n)$ or long-run mean-squared error (MSE) are explained in the subsequent section.
\begin{figure}[htb]
    \centering
    \includegraphics[width=4in]{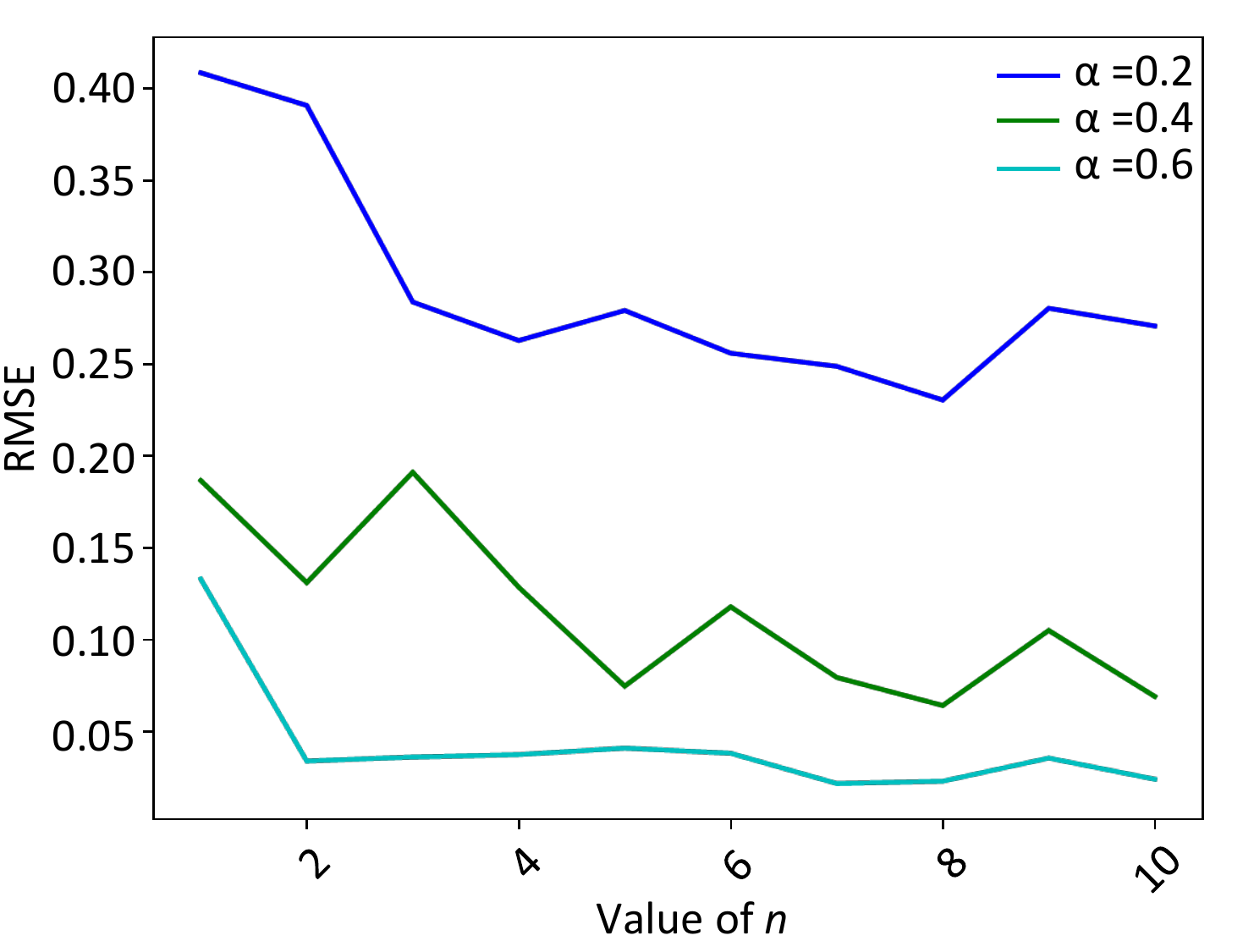}
    \caption{Obtained RMSE for different values of $n$ and fixed $\alpha$ on RW.}
    \label{fig:converged_RMSE_diff_n}
    \vspace{-3pt}
\end{figure}
\begin{figure}[htb]
    \centering
    \includegraphics[width=4in]{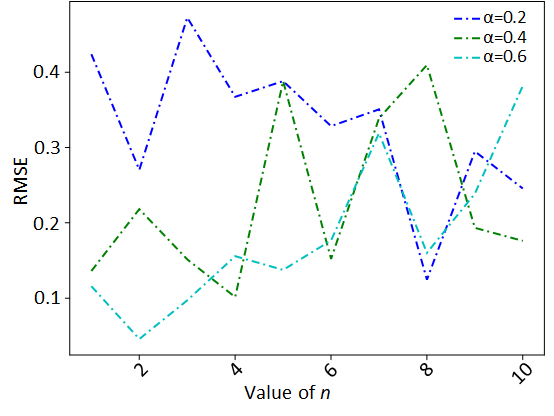}
    \caption{Obtained RMSE for different values of $n$ and fixed $\alpha$ on GW.}
    \label{fig:converged_RMSE_diff_n_app2}
\end{figure}

In Figs.~  \ref{fig:converged_RMSE_diff_n} and \ref{fig:converged_RMSE_diff_n_app2}, we show the effect of $n$ and $\alpha$ on the root mean squared error (RMSE) obtained upon convergence of Algorithm \ref{n_step_TD} for each value of $n$. These are shown for two different RL benchmark environments, namely the Random Walk (RW) and the Grid-world (GW), respectively, see  Section \ref{exp_results} for details. It is evident from these figures that the RMSE does not exhibit a monotonic behaviour and thus one cannot a priori predict the optimal value of $n$. These plots will also be seen to confirm the correctness of the results obtained by our proposed algorithm in Section \ref{exp_results}.

\subsection{The Proposed SDPSA Algorithm}

For our algorithm, we consider a deterministic sequence of perturbations whereby the perturbation variable $\Delta$ used to perturb the parameter is set equal to $+1$ on every even iteration ($k=0,2,4,\ldots$) and to $-1$ on every odd iteration ($k=1,3,5,\ldots$). 
Note that the state-valued process $\{S_{m}\}$ is a discrete Markov process that does not depend on the parameter $n$ as it only impacts the value estimation procedure and not the state evolution. Note that $n \in D \subset \mathcal{Z}^+$, a finite set of positive integers. 
Thus, let $D=\{1,2,\dots,L\}$ for some $L$ large enough and let $\bar D$ denote the closed convex hull of $D$, viz., $\bar D=[1, L]$. The SDPSA algorithm updates the parameter $n$ within the set $\bar D$. For $x\in \bar{D}$, $\Gamma(x) \in D$ is the projection of $x$ to $D$.

We assume that the state-process $\{S_m\}$ is ergodic Markov and takes values in the state space $S=\{0,1,\dots,|S|\}$, $|S|<\infty$. We denote the transition probability from state $x$ to state $y$ for the process $\{S_m\}$ to be $p_{xy}$, $x,y \in S$. Let $g_n(S_i)\dot{=}({\hat V_n}(S_i)-V_n(S_i))^2$ be the associated single-stage cost when the parameter value is $n$ and the state is $S_i$ at instant $i$. 
Thus, we consider the long-run mean-squared error $J(\cdot)$ or the long-run variance of $\hat V_n$, as the objective function to minimize. 
A desired global minimizer (i.e., $n^* \in D$) of $J(\cdot)$ will satisfy 
\begin{equation}
\label{eq:costFun}
    J(n^*)= \lim_{m \rightarrow \infty}\frac{1}{m}\mathbb{E} \left[\sum_{i=1}^m g_{n^*}(S_i)\right]=\min_{n \in D} J(n).
\end{equation}
The limit in (\ref{eq:costFun}) exists since $\{S_m\}$ is ergodic ({that also ensures the existence of a unique stationary distribution for the Markov chain under the given policy}) regardless of $n\in D$. 
Moreover, two projections $\bar \Gamma$ and $\Gamma$ are considered to respectively update the parameter in the set $[1,L]$ and determine the actual parameter to be used in the simulation. For the latter, we use a method of random projections to arrive at the parameter value to use from within $D$ given the update in $\bar{D}$ as follows: For $k\leq n \leq k+1$, $1\leq k<L$,
\begin{equation}
    \label{eq:projection}
    \Gamma(n):=
    \begin{cases}
      k, &  \text{w.p.} \ (k+1-n) \\
      k+1, & \text{w.p.} \ (n-k)
    \end{cases}
  \end{equation}
and for $n<1$ or $n>L$, we let
\begin{equation}
\label{eq:projection_bound}
    \Gamma(n):=
    \begin{cases}
      1, &  \text{if} \  n< 1\\
      L, & \text{if} \ n \geq L.
    \end{cases}
  \end{equation}

In (\ref{eq:projection}), we use the fact that $n=\beta k+(1-\beta)(k+1)$, $k \leq n \leq (k+1)$, for some $\beta\in [0,1]$, with $k,k+1\in D$. Thus, $\beta=(k+1-n)$ and $1-\beta = (n-k)$, respectively. The lower and upper bounds of the projected values of $n$ are $1$ and $L$, respectively, and are taken care of by (\ref{eq:projection_bound}). On the other hand, $\bar\Gamma$ is the operator that projects any point $n \in \mathbb{R}$ to the set $\bar{D}$ (the closed and convex hull of $D$) and is simply defined as $\bar\Gamma(n)=\min(L,\max(n,1))$. {From the foregoing, a natural way to define the value function estimate $ \hat V_n$ for $n\in \bar{D}$ and in particular, $k\leq n\leq k+1$, would be as follows:
\[
\hat{V}_n(x) := \beta \hat{V}^k(x) + (1-\beta)\hat{V}^{k+1}(x),
\]
where $\hat{V}^k$ and $\hat{V}^{k+1}$ respectively correspond to $\hat{V}_n$ in (\ref{hatV}) but with $n$ replaced by $k$ and $k+1$. }
Two step-size sequences, $\{a_m, m \geq 0\}$ and $\{b_m, m \geq 0\}$ {are employed that satisfy the following conditions:} 
$a_m,b_m >0$, $\forall m\geq 0$. Further, 
\begin{equation}
\label{stepSize1}
 \sum_{m} a_m=\sum_{m} b_m= \infty;
 \mbox{ } \frac{a_{k+1}}{a_{k}}
 \rightarrow 1 \mbox{ as } k\rightarrow \infty; 
 \end{equation}
\begin{equation}
\label{stepSize2}
 {\sum_m a_m^2 <\infty}; \sum_m b_m^2 <\infty; \mbox{ } {\lim_{m\rightarrow \infty} \frac{a_m}{b_m}=0}.
\end{equation}
These are standard conditions for step-size sequences in two-timescale stochastic recursions. 
In our work, $\{b_m\}$ is used to estimate both the value function as well as the average cost using a single simulation with a perturbed parameter, whereas $\{a_m\}$ is employed to update the parameter in the SDPSA algorithm. 

\begin{algorithm}[t]
\caption{The SDPSA Algorithm}
\label{SPSA}
{\textbf{Input:} Initial value of $n=n_0$, parameters $\nu$, $\delta$, $\Delta$.}


 {\textbf{Output:} Optimal value of $n=n^*$.}
 
$\Delta = +1$ for $k= 0,2,4,6,8,\ldots$

~~~~~~$= -1$ for $k= 1,3,5,7,\ldots$
           
For {$k=0,1,2\dots$}{

~~~~~~$ {n:=}\min(L,\max(1,\Gamma(n + \delta \Delta)))$

~~~~~~{Call Algorithm \ref{n_step_TD} to obtain $J(n)$, $\hat{\dot J}(n) :=\frac{J(n)}{\delta \Delta}$.}

~~~~~~$n \leftarrow n- \nu * {{\hat{\dot J}}}(n)$ ({use (\ref{eq:re_n_update_rule}) $-$ (\ref{TD_state_value_update})} )}

Return optimal $n$
\end{algorithm}

In the proposed SDPSA algorithm (see Algorithm \ref{SPSA}), $n$ is initialized to some $n_0\in D$. 
The SDPSA, unlike regular SPSA, utilizes a deterministic perturbed sequence (i.e., $ \{\Delta_m\}$) with $\Delta_m=+1$ (resp.~$-1$) for $m$ even (resp.~odd). This choice for the perturbation sequence ensures that the bias terms in the zeroth-order gradient estimates cancel cyclically and so the bias in the algorithm does not grow as the recursions progress. 
 The projected parameter $\Gamma(n_m+ \delta \Delta_m)$ is calculated and used for the value function estimate and that in turn gives the error term for the TD($n$) update. 
Let $n_m$ denote the value of $n$ in the $m$th update of the recursion, $Y_m$ be the running estimate of the long-run average cost and $V_m(i)$ be the $m$th update of the value of state $i$.
The following is the complete set of updates in the algorithm:
\begin{eqnarray}
\label{eq:re_n_update_rule}
n_{m+1}&=& \bar{\Gamma}\left(n_{m}-a_m \frac{Y_{m+1}}{\delta \Delta_m}\right),\\
\label{eq:re_grad_J}
    Y_{m+1}&=&Y_m+b_m \left( {g_{n_m}}(S_m) - Y_m \right),\\
\label{TD_state_value_update}
V_{m+1}(i) &=& V_m(i) + b_mI_{S_m}(i)({\hat V_{n_{m}}}(i) - V_m(i)), 
\end{eqnarray}
$i\in S$. 
Here $I_{S_m}(i) =+1$ if $S_m=i$ and $0$ otherwise, denotes the indicator function. Note again that the above algorithm performs a one-simulation zeroth-order stochastic gradient search and thus is amenable to online implementations unlike regular SPSA that requires two simulations with different perturbed parameters. 

\section{Asymptotic Convergence Analysis}
\label{conv_analysis}
In this section, we present a detailed asymptotic convergence analysis of the recursions (\ref{eq:re_n_update_rule})-(\ref{TD_state_value_update}). We first give an overview of the results and how they connect with one another. Lemma~\ref{lemma3} below sets the stage for the subsequent analysis. It shows that even though the function $J(n)$ is Lipschitz continuous in $n$, its derivative is discontinuous. Using a suitable two-timescale stochastic approximation argument, Proposition~\ref{prop1} shows the convergence of (\ref{eq:re_grad_J}) and (\ref{TD_state_value_update}) for a given value of $n_m$ (assumed constant as it operates on a slower timescale). We also show here that these recursions are stable. The resulting ODE for our algorithm is (\ref{anotherode2}). However, because its RHS is discontinuous (from Lemma~\ref{lemma3}), we identify a suitable set-valued map $H(n)$ that we show in Lemma~\ref{lem2} is Marchaud (see Definition~\ref{def1}). This property helps bring in the desired regularity on the set-valued map $H(n)$ that defines the DI (\ref{di}). Finally, we show in Theorem~\ref{theorem1} that the slower recursion (\ref{eq:re_n_update_rule}) asymptotically tracks the limit points of the DI (\ref{di}). Subsequently, we identify in Remark~\ref{rem1} the points of convergence of the algorithm.
\vspace{5pt}
\begin{lemma}
\label{lemma3}
 {$J(n)$ is a Lipschitz continuous function in $n\in \bar{D}$. Further, its derivative is piecewise Lipschitz continuous on intervals $[k,k+1)$, $1\leq k\leq L$ but discontinuous in general with points of discontinuity in the set $D$.} 
\end{lemma}
\vspace{-5pt}
\begin{proof}
Recall from (\ref{eq:projection})-(\ref{eq:projection_bound}) that $n_m\in \bar{D}$ satisfies $n_m=\beta k + (1-\beta) (k+1)$ for $\beta{=(k+1-n_m)} \in [0,1]$,  $k \leq n_m\leq (k+1)$. 
For a given (fixed) parameter $n_m \equiv n\in \bar{D}$, let $V_{(n)}$ 
denote the converged value of the value estimate $V_m$. This will correspond to $V_{(n)}(x)= E[\hat{V}_{n_m}(x)]$ for any $x\in S$ and can be obtained by running (9) alone keeping $n_m$ fixed.  
Then
\[
V_{{(n)}}(x) = \beta V^{k}(x) + (1-\beta) V^{k+1}(x),
\]
where $V^{k}(x)$ and $V^{k+1}(x)$ are respectively the (converged) value functions obtained when using $k$-step TD and $(k+1)$-step TD respectively. From (\ref{hatV}), it is easy to see that $V^k(x) = V^{k+1}(x)$, $\forall x\in S, k\in D$ since the solutions to the $k$-step Bellman equation or the $(k+1)$-step Bellman equation will be the same for any given policy. 
Thus, $V_{(n)}(x) = V^k(x) \stackrel{\triangle}{=} V(x)$, $\forall n\in \bar{D}$, $\forall k\in D$. Thus, $V_{(n)}(x)$ is a constant function of $n$ and is therefore trivially Lipschitz continuous.

{Now let $\hat{V}_{(n)}(s)$ denote the following estimate of ${V}_{(n)}(s)$ for $n\in \bar{D}$. In other words,
\[
\hat{V}_{(n)}(s) = (k+1-n) \hat{V}^k(s) + (n-k)\hat{V}^{k+1}(s),\]
where, with $S_0=s$,
\[
\hat{V}^k(s) = \sum_{j=1}^{k} \gamma^{j-1}R_j + \gamma^k V(S_k),
\]
\[
\hat{V}^{k+1}(s) = \sum_{j=1}^{k+1} \gamma^{j-1}R_j + \gamma^{k+1} V(S_{k+1}),
\]
respectively. The difference of $\hat{V}^k(\cdot)$ (or $\hat{V}^{k+1}(\cdot)$) from $\hat{V}_n(\cdot)$ in (\ref{hatV}) is that $V_n(\cdot)$ there which was an estimate similar to (9) that is now replaced by it's converged value (i.e., the true value function) $V(\cdot)$.
From the definition of $J(n)$ as the long-run average of $g_n(S_i)$, it follows that for $k\leq n \leq k+1$, with $d(s), s\in S$ being the stationary distribution of the Markov chain being in state $s\in S$, we have
\[
J(n) = \sum_s d(s) E[(\hat{V}_{(n)}(s) - V_{(n)}(s))^2],
\]
where the expectation $E[\cdot]$ is taken over the joint distribution of states $S_1,\ldots,S_{[n]}$, where $[n]$ is the randomly projected time instant to one of the integers $k$ or $k+1$ (where $k\leq n\leq k+1$). Thus,
\[
J(n) = (k+1-n)^2\sum_s d(s) Var(\hat{V}^k(s)) \]
\[
+ (n-k)^2 \sum_s d(s) Var(\hat{V}^{k+1}(s)) \]
\[
+ 2(k+1-n)(n-k)\sum_sd(s) Cov(\hat{V}^k(s), \hat{V}^{k+1}(s)).
\]
Here $Var(X)$ stands for the variance of $X$ and $Cov(X,Y)$ denotes the covariance of $X$ and $Y$,  respectively. 
Let $J_k(n) \stackrel{\triangle}{=} J(n)|_{[k,k+1]}$ be the restriction of $J(n)$ to the interval $[k,k+1]$. Then, from the above, for $k\leq m,l <k+1$, using the fact that the single-stage reward and the value function $V(s)$ as well as the estimates $\hat{V}^k(s)$ are uniformly bounded, and furthermore $1\leq m,l \leq L$,
\[|J_k(m)-J_k(l)| \leq M_k|m-l|,\]
where $M_k>0$ is the Lipschitz constant for $J_k(\cdot)$ on the interval $[k,k+1]$, the subscript $k$ indicating the dependence of the constant on the aforementioned interval.

Further, for $m\in [k,k+1]$,
\[\lim_{m\downarrow k} J_k(m) = \sum_s d(s)Var(\hat{V}^k(s)).\] Now, 
for $l\in [k-1,k]$,
\[\lim_{l\uparrow k} J_{k-1}(m) = \sum_s d(s)Var(\hat{V}^k(s)).\]
Thus, $J(\cdot)$ is piecewise Lipschitz continuous on the intervals $[k,k+1]$ and is continuous at the boundary points of such intervals.
Now observe that 
\[
|\max_{k\in D} J_k(m) - \max_{k\in D} J_k(l)|
\leq \max_{k\in D} |J_k(m)-J_k(l)| \]\[
\leq M|m-l|,
\]
where $M= \max_{k\in D} M_k$. Thus, $J(\cdot)$ is Lipschitz continuous in $\bar{D}$. Now,
\[
\frac{dJ(n)}{dn} = 
 -2(k+1-n)\sum_s d(s) Var(\hat{V}^k(s)) \] 
\[
+ 2(n-k) \sum_s d(s) Var(\hat{V}^{k+1}(s))
\]
\[
+ (4k+2-4n) \sum_sd(s) Cov(\hat{V}^k(s),\hat{V}^{k+1}(s)).
\]
Now for, $m,l\in [k,k+1)$ we have
\[
\left|\frac{dJ(n)}{dn}|_{n=m}-\frac{dJ(n)}{dn}|_{n=l} \right| \]\[\leq
2|m-l|(\max_{s} Var(\hat{V}^{k+1}(s))+
\max_{s} Var(\hat{V}^{k}(s)))\]
\[
+4|m-l|\max_s Cov(\hat{V}^{k}(s), \hat{V}^{k+1}(s)).
\]
Since, the single stage rewards are uniformly bounded almost surely and so also are the value functions, there exists a constant $K_1>0$ such that
\[
\left|\frac{dJ(n)}{dn}|_{n=m}-\frac{dJ(n)}{dn}|_{n=l} \right| \leq
K_1|m-l|,
\]
for all $m,l\in [k,k+1)$. 
Thus, ${\displaystyle \frac{dJ(n)}{dn}}$ is Lipschitz continuous on $[k,k+1)$ for any $1\leq k\leq L$. Now observe that
\[
\lim_{n\downarrow k} \frac{dJ(n)}{dn}
= -2\sum_s d(s) (Var(\hat{V}^k(s)) -
Cov(\hat{V}^k(s), \hat{V}^{k+1}))
\]
\[
\not= \lim_{n\uparrow k} \frac{dJ(n)}{dn}.
\]
The latter limit above can be obtained by taking $n\in [k-1,k)$ and letting $n\rightarrow k$.}
%
\end{proof}
We now analyse the two-timescale recursions (\ref{eq:re_n_update_rule})-(\ref{TD_state_value_update}) under the step-size requirements (\ref{stepSize1})-(\ref{stepSize2}). 
{Given any real-valued continuous function $f(x)$}, with $x\in \bar{D}$, let 
${\displaystyle 
\hat{\bar{\Gamma}}(f(x)) = \lim_{\beta\rightarrow 0} \left(
\frac{\bar{\Gamma}(x+ \beta f(x))-x}{\beta}
\right)}.$
Since $\bar{D}$ is a closed and convex set, we have that for any $x\in \bar{D}$, $\hat{\bar{\Gamma}}(f(x))$ is well defined and unique. 
Let $\mathbb{D}$ be a diagonal matrix containing the stationary probabilities $d(i),i\in S$ as the diagonal elements. Let $V\stackrel{\triangle}{=} (V(1),\ldots,V(|S|))^T$. 

Consider now the following system of ordinary differential equations (ODEs) (that correspond to the fast timescale):
\begin{eqnarray}
\label{ode1}
   \dot n (t) &=& 0,\\
\label{ode2}
   \dot Y(t)&=& J(\bar \Gamma(n(t)+\delta \Delta(t))) -Y(t),\\
    \label{ode2-1}
    \dot{V}(t) &=& \mathbb{D}(V_{n(t)}(t) - V(t)).
\end{eqnarray}
In the light of \eqref{ode1}, $n(t)\equiv n$, $\forall t$, hence the ODEs \eqref{ode2}-\eqref{ode2-1} can be rewritten as
\begin{eqnarray}
\label{ode3}
   \dot Y(t)&=&J(\bar \Gamma(n+\delta \Delta(t))) -Y(t),\\
   \label{ode3-1}
\dot{V}(t) &=& \mathbb{D}(V_{n} - V(t)).
\end{eqnarray}
Here $\Delta(t)=\Delta_m,$ for $t\in [c(m),c(m)+a(m)]$, where, $c(m)=\sum_{i=0}^{m-1}a(i)$, $m\geq 1$. Now \eqref{ode3} has $Y^*=\lambda(n) \equiv J(\bar{\Gamma}(n+\delta \Delta_m))$ as {its} unique globally asymptotically stable equilibrium. Similarly, $V^*=V_n$ is the unique globally asymptotically stable equilibrium of \eqref{ode3-1}. By Lemma~\ref{lemma3}, $\lambda(\cdot)$ is Lipschitz continuous.

%

\vspace{5pt}
\begin{proposition}
\label{prop1}
The following hold:
\begin{itemize}
    \item[(a)] $\|Y_m-J(\bar{\Gamma}(n+\delta\Delta_m))\| \rightarrow 0$ a.s. as $m\rightarrow \infty$,
    \item[(b)] $\|V_m - V_{(n)}\| \rightarrow 0$ a.s. as $m\rightarrow \infty$.
\end{itemize}
\end{proposition}
\vspace{-5pt}

\begin{proof}
 Consider the recursion \eqref{eq:re_grad_J}. Note that since the state space $S$ is finite, the rewards\\
 $|R(S_t,\pi(S_t),S_{t+1})| \leq B$ w.p.1 for some $B<\infty$, i.e., the rewards are a.s.~uniformly bounded. 
As noted previously, it is easy to see that
$\sup_s \sup_n |V_n(s)| <\infty$  and 
$\sup_s \sup_n |g_n(s)| <\infty$.
Now observe that from the first part of conditions \eqref{stepSize2}, both $a_m,b_m\rightarrow 0$ as $m\rightarrow\infty$. Then $\exists n_0\geq 1$ such that for all $n>n_0$, $0<b_n<1$ and $Y_{m+1}$ in recursion \eqref{eq:re_grad_J} is a convex combination of $Y_m$ and $g_n(S_m)$ a uniformly bounded quantity. Thus, the recursion \eqref{eq:re_grad_J} remains uniformly bounded almost surely, i.e., that
${\displaystyle 
\sup_n |Y_m| <\infty}$ a.s. Similarly, the recursion (\ref{TD_state_value_update}) remains uniformly bounded as well. The ODE \eqref{ode3} has $Y^*(n) = J(\bar \Gamma(n+{\delta} \Delta_m))$ as {its} unique globally asymptotically stable equilibrium. Likewise \eqref{ode3-1} has $V_{(n)}$ as {its} unique globally asymptotically stable equilibrium.
The claims in both (a) and (b) above now follow from Theorem 8.2, Chapter 8, of \cite{borkar-book}.
\end{proof}

We now consider the slower timescale recursion (7). 
The ODE on the slower timescale is the following:
\begin{equation}
\label{anotherODE}
    \dot n(t)=\hat{\bar{\Gamma}}\left(-\frac{J(\bar{\Gamma}(n(t)+\delta\Delta(t)))}{\delta\Delta(t)}\right).
\end{equation}
Consider the case when $n_m$ lies in the interior of $\bar{D}$. Then for $\delta>0$ sufficiently small (i.e., $\delta<\delta_0$ for some $\delta_0>0$), we will have that $\bar{\Gamma}(n_m+\delta\Delta_m) = n_m +\delta\Delta_m$. 
 Assuming that $\dot{J}(n)$ is continuous, a Taylor's expansion 
around the point $n$ would give 
\begin{equation}
\label{eq:taylorExp}
   J(n+\delta\Delta_m)
   =J(n)+ \delta\Delta_{m} \dot{J}(n)+o(\delta).
\end{equation}
Hence,
\begin{equation}
    \frac{J(n+\delta \Delta_m)}{\delta\Delta_m}
    =\frac{J(n)}{\delta\Delta_m}+ \dot{J}(n)+O(\delta).
\end{equation}
Thus the bias in the gradient estimation would be as follows:
\begin{equation}
    \label{eq:bias_grad}
    \frac{J(n+\delta\Delta_m)}{\delta \Delta_m} - \dot{J}(n) \\
    =\frac{J(n)}{\delta} \left[\frac{1}{\Delta_m}\right]+ O(\delta).
\end{equation}
It is easy to see that for any $k\geq 0$,
$\sum_{m=k}^{k+1} \frac{1}{\Delta_m}=0$, i.e., the summation becomes zero over any cycle of length two. It can now be shown using \eqref{stepSize1}, in a similar manner as Corollary 2.6 of \cite{bhatnagar2003} that
\[
\|\sum_{m=k}^{k+1} \frac{a(m)}{a(k)}\frac{1}{\Delta_m}J(n)\| \rightarrow 0 \mbox{ as } k\rightarrow\infty.
\]
Thus for $\delta>0$ sufficiently small, it will follow from Theorem 1, Chapter 2 of \cite{borkar-book} that the slower recursion (7) governing $n_m$ will converge almost surely to a small neighborhood of the set $\bar{K} = \{n | \hat{\bar{\Gamma}}(\dot{J}(n)) =0\}$.  
Note however from Lemma~\ref{lemma3} above that $\dot{J}(n)$, $n\in \bar{D}$ is Lipschitz continuous on the intervals $[k,k+1)$ but is discontinuous in general at the boundary points $k \in D$. This means that one cannot approximate the ODE (\ref{anotherODE}), for $\delta>0$ small, with the ODE
\begin{equation}
    \label{anotherode2}
    \dot{n}(t) = \hat{\bar{\Gamma}}(-\dot{J}(n)),
\end{equation}
because (\ref{anotherode2}) has a discontinuous RHS. A natural way to deal with such systems is to consider a differential inclusion (DI) limit that involves point-to-set maps \cite{benaim, RB1, AB4, YB1, YB2} instead of (\ref{anotherode2}). We present the construction of the DI below.  

We define a set-valued map $H(n)$, see \cite{NY, BS}, as follows:
\[
H(n) = \cap_{\eta>0}\cap \overline{co}(\{\hat{\bar{\Gamma}}(-\dot{J}(m)) | \|m-n\|<\eta\}).
\]
Here $\overline{co}(A)$ denotes the closed convex hull of the set $A$.
Note that the operator $\hat{\bar{\Gamma}}$ ensures that the trajectories of the DI stay within $\bar{D}=[1,L]$ and in particular for $n\in (1,L)$, $\hat{\bar{\Gamma}}(-\dot{J}(n)) = -\dot{J}(n)$. Further, for $n=1$, $\hat{\bar{\Gamma}}(-\dot{J}(n)) = -\dot{J}(n)$ if $\dot{J}(n) <0$ (since then $\bar{\Gamma}(n-\beta \dot{J}(n)) = n-\beta \dot{J}(n)$ for $\beta>0$) small enough, else it equals 0 (since 
then $\bar{\Gamma}(n-\beta J(n)) = n$). Similarly, for $n=L$, $\hat{\bar{\Gamma}}(-\dot{J}(n)) = -\dot{J}(n)$ if $\dot{J}(n) > 0$,
else it equals 0. Thus, the set-valued map $H(n)$ in our setting corresponds to the following: $H(n)=-\dot{J}(n)$, for $n\in (k,k+1)$ with $k,k+1\in D$, and $H(n) = [\alpha_k,\beta_k]$, for $n=k\in [2,L-1]$, where $\alpha_k$ (resp.~$\beta_k$) is the lower (resp.~upper) limit of $\dot{J}(n)$ at $n=k$. Furthermore, for $n=1$ and $n=L$, we still let $H(n)=[\alpha_k,\beta_k]$ with $k=1$ or $L$ if $0\in H(n)$. Else, we take the closed convex hull of the points $0,\alpha_k,\beta_k$ when $k=1$ or $k=L$. For ease of exposition, we continue to call $H(1) = [\alpha_1,\beta_1]$ and $H(L) = [\alpha_L, \beta_L]$, respectively, assuming that $0$ is an element of these sets.

\vspace{5pt}
\begin{definition}
\label{def1}
We say that a set-valued map $H:\mathcal{R} \rightarrow \{\mbox{subsets of }\mathcal{R}\}$ is Marchaud if it satisfies the following conditions:
\vspace{-5pt}
\begin{enumerate}
    \item $H(n)$ is compact and convex for every $n$.
    \item $H(n)$ satisfies a linear growth condition, i.e., that
    \[
    \max_{l\in H(n)} |l| \leq L'(1+|n|),
    \]
    for some $L'>0$. 
    \item $H(n)$ is upper-semicontinuous, i.e., it's graph $\{(n,m) | m\in H(n)\}$ is closed.
\end{enumerate}
\end{definition}
Consider now the differential inclusion (DI): 
\begin{equation}
    \label{di}
    \dot{n}(t) \in H(n(t)).
\end{equation}
An important consequence of $H(\cdot)$ being Marchaud is that it would then follow from \cite{aubin} that every solution to the DI (\ref{di}) will be absolutely continuous.

\vspace{10pt}
\begin{lemma}
    \label{lem2}
    The set-valued map $H(n)$ is Marchaud.
\end{lemma}
\vspace{-10pt}
\begin{proof}

Note that $H(n)$ is the singleton $H(n)=\{\dot{J}(n)\}$ for $n\in (k,k+1)$, where $k,k+1\in D$. Thus, for $n$ as above, $H(n)$ is trivially compact and convex. Further, for $n=k\in D$, $H(n) =[\alpha_k,\beta_k]$, where $-\infty< \alpha_k < \beta_k < \infty$. Clearly $H(n)$ for $n=k$ is compact and convex as well.

Next, we show point-wise boundedness of the set-valued map. Note that for $n\in (k,k+1)$, with $k,k+1\in D$, 
\[
\sup_{l\in H(n)} |l| = |\dot{J}(n)| \leq K'(1+n).
\]
This follows because $H(n)=\{\dot{J}(n)\}$ for $n\in (k,k+1)$. Moreover, from Lemma~\ref{lemma3}, $\dot{J}(n)$ is Lipschitz continuous in $(k,k+1)$. Thus, for $n,n_0 \in (k,k+1)$, and some $L'>0$,
\[
|J(n)|-|J(n_0)| \leq |J(n)-J(n_0)| \leq L'|n-n_0|.
\]
Thus,
\[
|J(n)| \leq |J(n_0)| +L'n_0 +L'n \leq L''(1+n),
\]
where $L'' = |J(n_0)|+L'n_0$. 
Note also that for $n=k$, $H(n) = [\alpha_k,\beta_k]$. Hence, ${\displaystyle \sup_{l\in H(n)} |l| \leq K''(1+n)}$ since $[\alpha_k,\beta_k]$ is compact and $k\in D$, a finite set.
   
Finally, we show that $H(n)$ is upper semi-continuous. The non-trivial case here is when a sequence $\{n_l\}$ of points in $\bar{D}$ converges to an integer (point) $k\in D$, i.e., $n_l\rightarrow k$ as $l\rightarrow\infty$. Thus if $n_l \not\in D$, $H(n_l) =\{\dot{J}(n_l)\}$. Then the only points $y_l \in H(n_l)$ that one may consider are $y_l= \dot{J}(n_l)$, $\forall n_l$. The limit point $y$ of the sequence $\{y_l\}$ will be either $\alpha_k$ or $\beta_k$, both of which are contained in $H(k)$. Thus, $H(n)$ is upper semi-continuous as well. The claim follows.
\end{proof}

\begin{theorem}
\label{theorem1}
The sequence $\{n_m\}$ obtained from the SDPSA algorithm satisfies $n_M\rightarrow P$ almost surely as $m\rightarrow\infty$, where $P$ is an internally chain transitive set of the DI (\ref{di}).
\end{theorem}

\begin{proof}
Note that we can rewrite
(\ref{eq:re_n_update_rule}) as follows:
\begin{equation}
\label{eq:re_n_update_rule1}
n_{m+1}= \bar{\Gamma}\left(n_{m}-b_m \left(\frac{J(\bar{\Gamma}(n_m +\delta\Delta_m))}{\delta\Delta_m}\right)\right).
\end{equation}
From the foregoing, the above is analogous to 
\begin{equation}
\label{perturbed}
n_{m+1} = \bar{\Gamma}(n_m -b_m (z(n_m) + O(\delta))),
\end{equation}
where $z(n_m) \in H(n_m)$. In particular, let $z(n_m) = \dot{J}(n_m)$ for $n_m \in (k,k+1)$, $k,k+1\in D$. It follows from Theorem 3.6 of \cite{benaim} that the limit set\\
${\displaystyle L(y) = \cap_{t\geq 0} \overline{\{y(s) | s\geq t\}}}$ of any bounded perturbed solution $y(\cdot)$ to the DI (\ref{di}) is internally chain transitive. Further, (\ref{perturbed}) by itself is a bounded and perturbed solution to the DI (\ref{di}). The claim follows. 
\end{proof}

\begin{remark}
\label{rem1}
Note from Theorem~\ref{theorem1} that if $\hat{\bar{\Gamma}}(-\dot{J}(n^*))=0$ for some $n^*\in C\subset \bar{D}$, then $0\in H(n^*)$ and the recursion (\ref{perturbed}) will converge to the largest chain transitive invariant set contained in $C$. From the foregoing, there are at least two points in $D$, namely $n=1$ and $n=L$ for which $0\in H(n)$. Thus, in general, if the algorithm does not converge to a point in the set
$D^o=\{2,3,\ldots,L-1\}$, it will converge to either $n=1$ or $n=L$.
\end{remark}

\section{Experiments and Results}
\label{exp_results}
The goal of our experiments is to verify the correctness and efficiency of our proposed SDPSA algorithm to achieve the optimal value of $n$ to use in the case of the $n$-step TD algorithm starting from an arbitrary initial value of the same. For our implementations, we consider two different RL benchmark environments, namely the Random Walk (RW) (cf.~Chapter 7 of \cite{RL_Book}) and the Grid-world (GW) (an open-source benchmark from Farama Foundation \cite{MinigridMiniworld23}), comprising of $21$ and $256$ states, respectively. 
 
In our experiments, the proposed SDPSA algorithm is implemented and our objective is to obtain the optimal value of $n$ for any given value of the step-size parameter $\alpha$. As mentioned earlier, $\alpha>0$  is kept fixed to mimic the setting of \cite{RL_Book}. We run our algorithm for different values of $\alpha$ and initial values of $n$. For the experiments on RW, we observe that for each $\alpha$, the value of $n$ to which our algorithm converges is the same as seen from Fig.~7.2 of \cite{RL_Book}. We use the same set of step-sizes $\alpha$ (i.e., 0.6, 0.4, and 0.2) as in \cite{RL_Book}. However, unlike the procedure in \cite{RL_Book} where the true value function is computed to get the RMSE, we use the TD-error itself to estimate the RMSE,  making our SDPSA algorithm fully incremental and model-free.

\begin{figure}[b]
\centering
\vspace{-45pt}
  \subfloat[]{
	\begin{minipage}[c][1\width]{
	   0.5\textwidth}
	   \centering
        \includegraphics[width=0.86\textwidth]{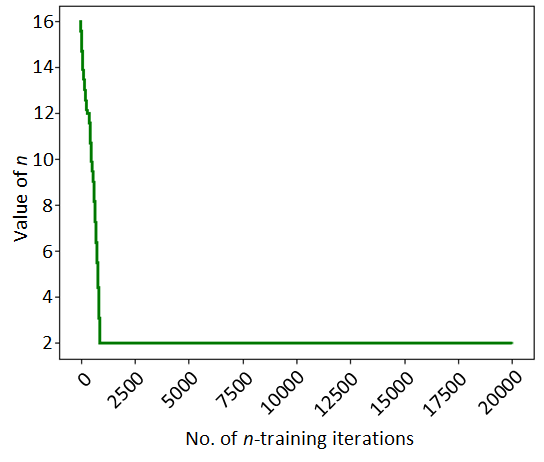}
        \vspace{-50pt}
	\end{minipage}}	
  \subfloat[]{
	\begin{minipage}[c][1\width]{
	   0.5\textwidth}
        \hspace{-30pt}
	   \centering
	   \includegraphics[width=0.86\textwidth]{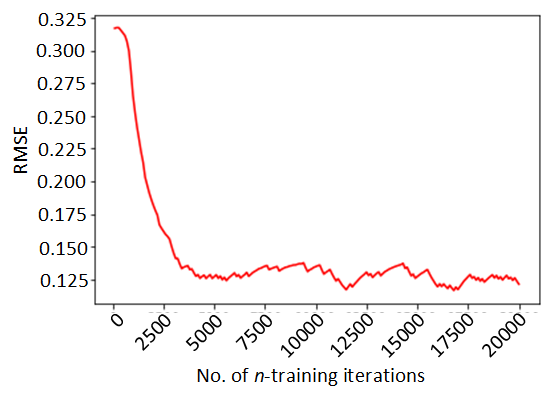}
    \vspace{-60pt}
	\end{minipage}}
 \vspace{-10pt}
\caption{(a) $n$-updates and (b) running RMSE values w.r.t. number of iterations, with initial $n=16$ and $\alpha=0.6$ on RW}.
\label{n16_pt6}
\vspace{-10pt}
\end{figure}

\begin{figure}[htb]
\centering
\vspace{-35pt}
  \subfloat[]{
	\begin{minipage}[c][1\width]{
	   0.5\textwidth}
	   \centering
        \includegraphics[width=0.86\textwidth]{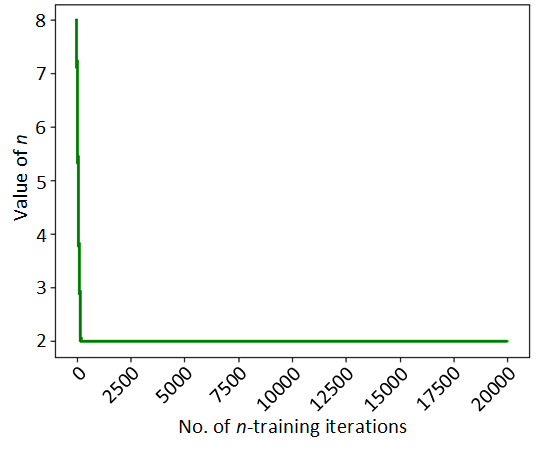}
        \vspace{-50pt}
	\end{minipage}}
  \subfloat[]{
	\begin{minipage}[c][1\width]{
	   0.5\textwidth}
        \hspace{-30pt}
	   \centering
	   \includegraphics[width=0.86\textwidth]{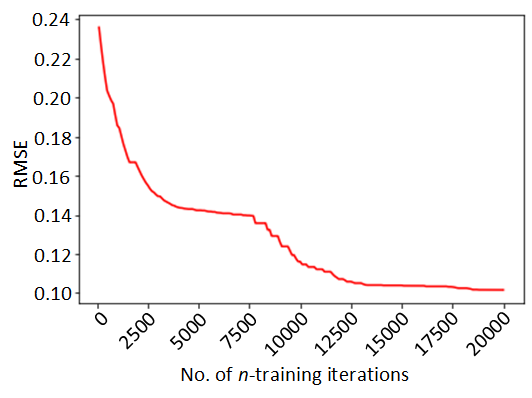}
    \vspace{-60pt}
	\end{minipage}}
 \vspace{-10pt}
 \caption{(a) $n$-updates and (b) running RMSE values w.r.t. number of iterations, with initial $n=8$ and $\alpha=0.6$ on RW}.
\label{n8_pt6}
\end{figure}

\begin{figure}[htb]
\centering
\vspace{-45pt}
  \subfloat[]{
	\begin{minipage}[c][1\width]{
	   0.5\textwidth}
	   \centering
        \includegraphics[width=0.86\textwidth]{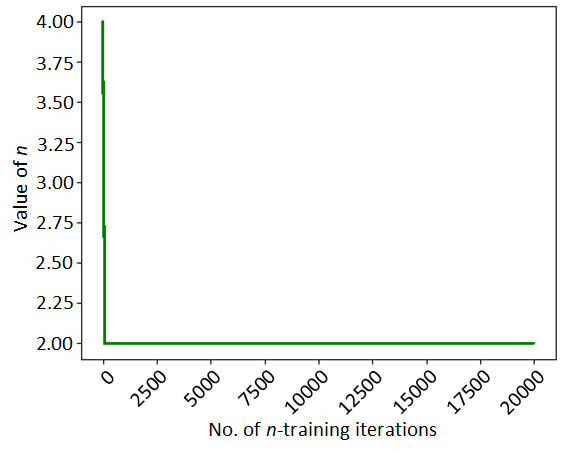}
        \vspace{-50pt}
	\end{minipage}}
  \subfloat[]{
	\begin{minipage}[c][1\width]{
	   0.5\textwidth}
        \hspace{-30pt}
	   \centering
	   \includegraphics[width=0.86\textwidth]{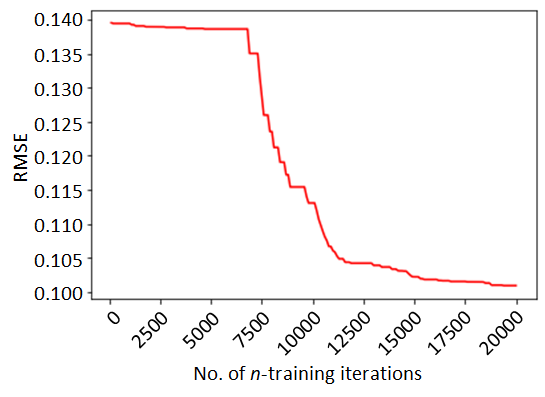}
    \vspace{-60pt}
	\end{minipage}}
 \vspace{-10pt}
 \caption{(a) $n$-updates and (b) running RMSE values w.r.t. number of iterations, with initial $n=4$ and $\alpha=0.6$ on RW}.
\label{n4_pt6}
\vspace{-10pt}
\end{figure}
\begin{figure}[htb]
\centering
\vspace{-35pt}
  \subfloat[]{
	\begin{minipage}[c][1\width]{
	   0.5\textwidth}
	   \centering
        \includegraphics[width=0.86\textwidth]{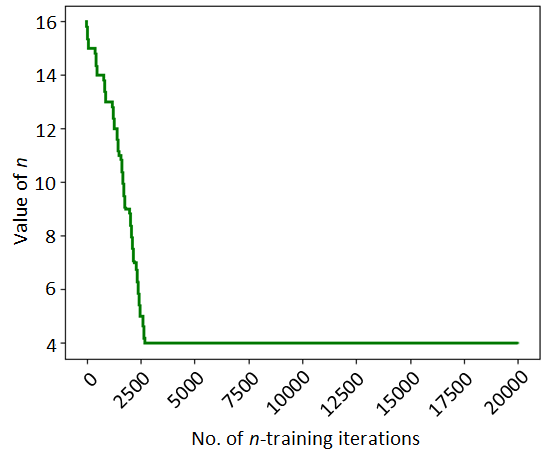}
        \vspace{-50pt}
	\end{minipage}}
  \subfloat[]{
	\begin{minipage}[c][1\width]{
	   0.5\textwidth}
        \hspace{-30pt}
	   \centering
	   \includegraphics[width=0.86\textwidth]{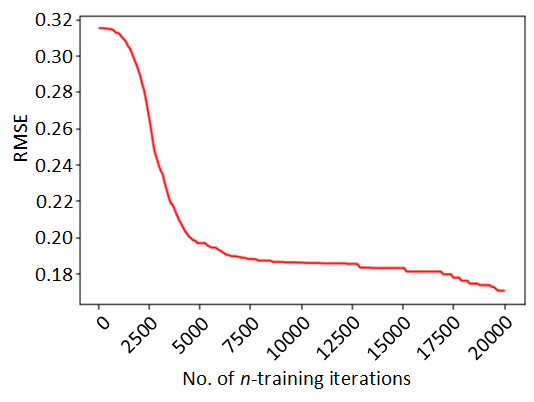}
    \vspace{-60pt}
	\end{minipage}}
 \vspace{-10pt}
 \caption{(a) $n$-updates and (b) running RMSE values w.r.t. number of iterations, with initial $n=16$ and $\alpha=0.4$ on RW}.
\label{n16_pt4}
 \vspace{-20pt}
\end{figure}

\begin{figure}[htb]
\centering
\vspace{-30pt}
  \subfloat[]{
	\begin{minipage}[c][1\width]{
	   0.5\textwidth}
	   \centering
        \includegraphics[width=0.86\textwidth]{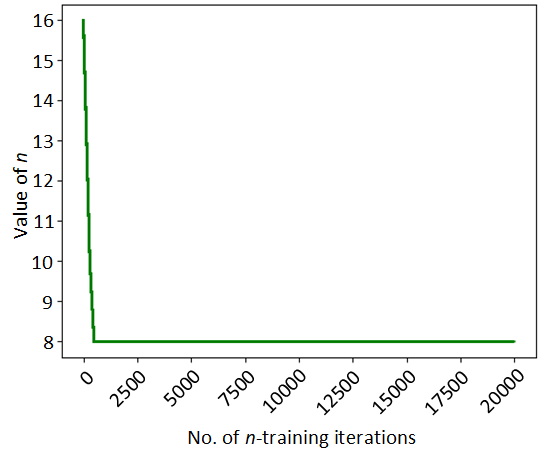}
        \vspace{-50pt}
	\end{minipage}}	 
  \subfloat[]{
	\begin{minipage}[c][1\width]{
	   0.5\textwidth}
        \hspace{-30pt}
	   \centering
	   \includegraphics[width=0.86\textwidth]{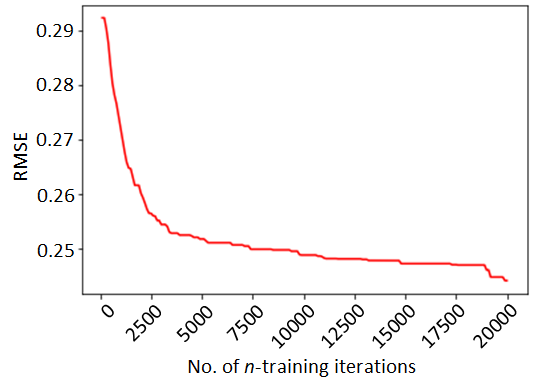}
    \vspace{-50pt}
	\end{minipage}}
\vspace{-10pt}
 \caption{(a) $n$-updates and (b) running RMSE values w.r.t. number of iterations, with initial $n=16$ and $\alpha=0.2$ on RW}.
\label{n16_pt2}
\vspace{-5pt}
\end{figure}
\begin{figure}[htb]
\centering
\vspace{-40pt}
  \subfloat[]{
	\begin{minipage}[c][1\width]{
	   0.5\textwidth}
	   \centering
        \includegraphics[width=0.86\textwidth]{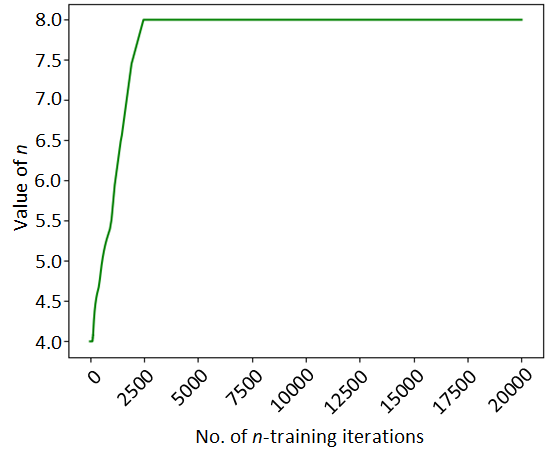}
        \vspace{-50pt}
	\end{minipage}}	 
  \subfloat[]{
	\begin{minipage}[c][1\width]{
	   0.5\textwidth}
        \hspace{-30pt}
	   \centering
	   \includegraphics[width=0.86\textwidth]{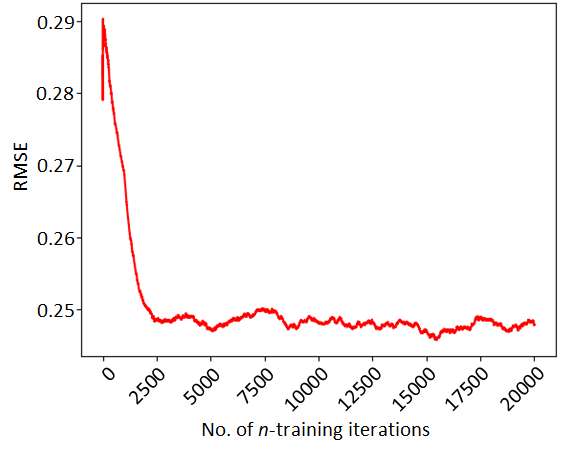}
    \vspace{-50pt}
	\end{minipage}}
 \vspace{-10pt}
 \caption{(a) $n$-updates and (b) running RMSE values w.r.t. number of iterations, with initial $n=4$ and $\alpha=0.2$ on RW}.
\label{n4_pt2}
\vspace{-5pt}
\end{figure}

\begin{figure}[t]
\centering
\vspace{-40pt}
  \subfloat[]{
	\begin{minipage}[c][1\width]{
	   0.5\textwidth}
	   \centering
        \includegraphics[width=0.86\textwidth]{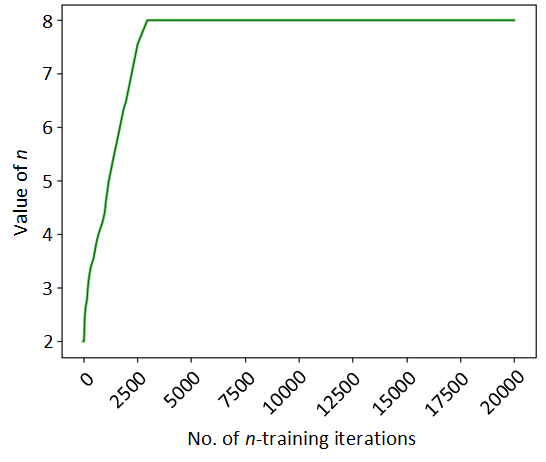}
        \vspace{-50pt}
	\end{minipage}}	 
  \subfloat[]{
	\begin{minipage}[c][1\width]{
	   0.5\textwidth}
        \hspace{-30pt}
	   \centering
	   \includegraphics[width=0.86\textwidth]{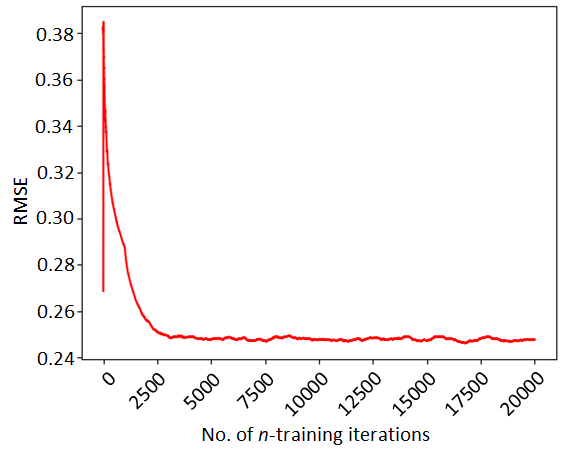}
    \vspace{-50pt}
	\end{minipage}}
  \vspace{-10pt}
 \caption{(a) $n$-updates and (b) running RMSE values w.r.t. number of iterations, with initial $n=2$ and $\alpha=0.2$ on RW}.
\label{n2_pt2}
\end{figure}

\begin{figure}[htb]
\centering
\vspace{-42pt}
  \subfloat[]{
	\begin{minipage}[c][1\width]{
	   0.5\textwidth}
	   \centering
        \includegraphics[width=0.92\textwidth]{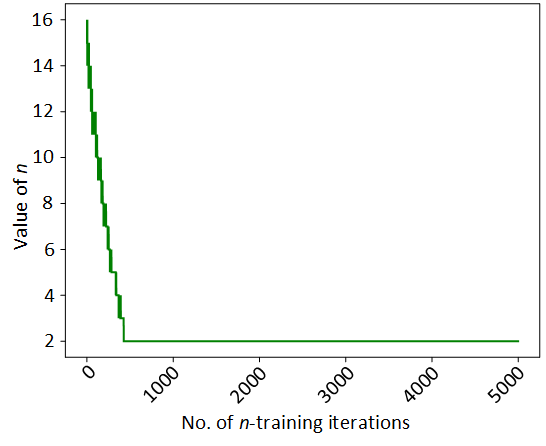}
        \vspace{-50pt}
	\end{minipage}}	
  \subfloat[]{
	\begin{minipage}[c][1\width]{
	   0.5\textwidth}
        \hspace{-30pt}
	   \centering
	   \includegraphics[width=0.92\textwidth]{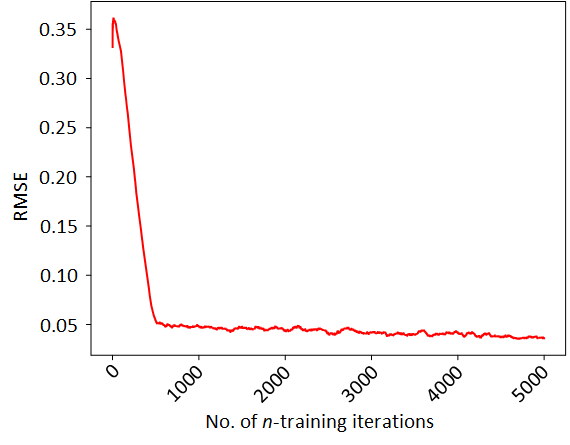}
    \vspace{-60pt}
	\end{minipage}}
 \vspace{-3pt}
\caption{(a) $n$-updates and (b) running RMSE values w.r.t. number of iterations, with initial $n=16$ and $\alpha=0.6$ on GW.}
\label{n16_pt6_app2}
\end{figure}
 \begin{figure}[htb]
\centering
\vspace{-50pt}
  \subfloat[]{
	\begin{minipage}[c][1\width]{
	   0.5\textwidth}
    \hspace{-20pt}
	   \centering
        \includegraphics[width=0.92\textwidth]{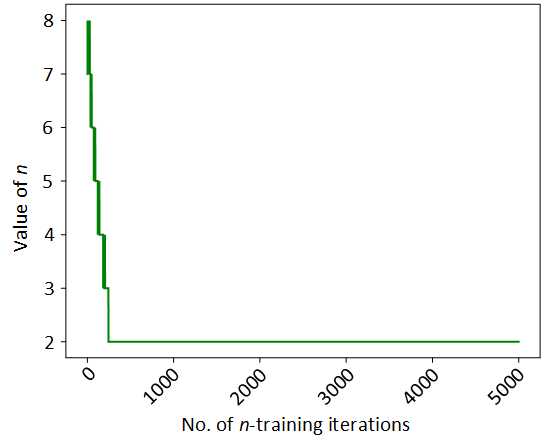}
        \vspace{-50pt}
	\end{minipage}}	
  \subfloat[]{
	\begin{minipage}[c][1\width]{
	   0.5\textwidth}
        \hspace{-30pt}
	   \centering
	   \includegraphics[width=0.92\textwidth]{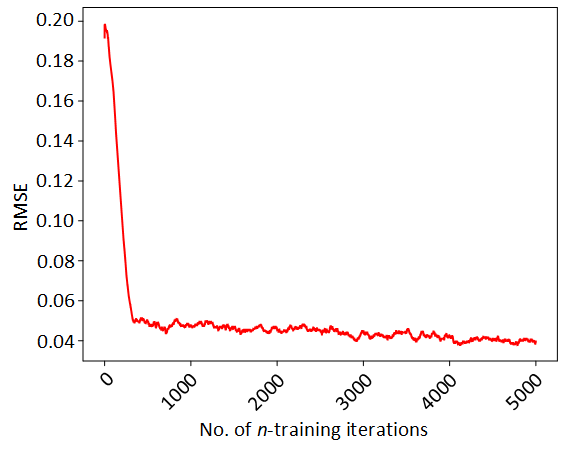}
    \vspace{-50pt}
	\end{minipage}}
 \vspace{-10pt}
\caption{{(a) $n$-updates and (b) running RMSE values w.r.t. number of iterations, with initial $n=8$ and $\alpha=0.6$ on GW.}}
\label{n8_pt6_app2}
\end{figure}

\begin{figure}[htb]
\centering
\vspace{-40pt}
  \subfloat[]{
	\begin{minipage}[c][1\width]{
	   0.5\textwidth}
    \hspace{-20pt}
	   \centering
        \includegraphics[width=0.92\textwidth]{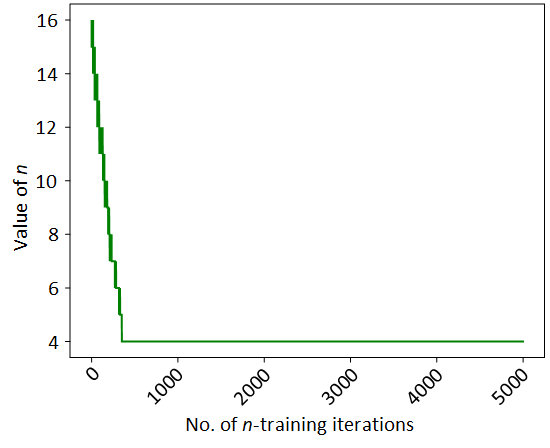}
        \vspace{-50pt}
	\end{minipage}}	
  \subfloat[]{
	\begin{minipage}[c][1\width]{
	   0.5\textwidth}
        \hspace{-30pt}
	   \centering
	   \includegraphics[width=0.92\textwidth]{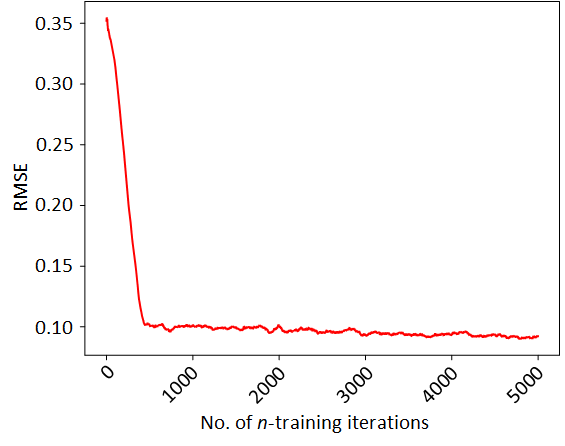}
    \vspace{-50pt}
	\end{minipage}}
 \vspace{-10pt}
\caption{{(a) $n$-updates and (b) running RMSE values w.r.t. number of iterations, with initial $n=16$ and $\alpha=0.4$ on GW.}}
\label{n16_pt4_app2}
\end{figure}
 \begin{figure}[htb]
\centering
\vspace{-45pt}
  \subfloat[]{
	\begin{minipage}[c][1\width]{
	   0.5\textwidth}
    \hspace{-20pt}
	   \centering
        \includegraphics[width=0.92\textwidth]{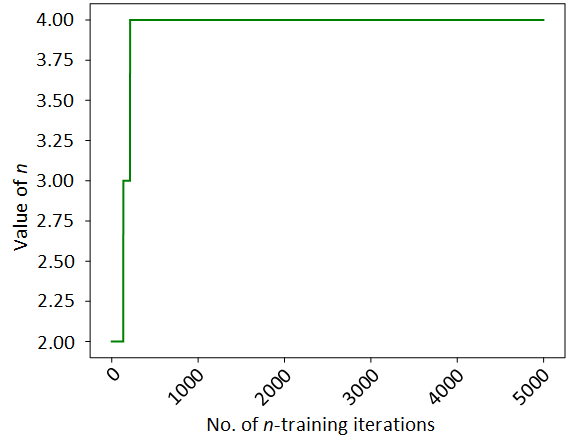}
        \vspace{-50pt}
	\end{minipage}}	
  \subfloat[]{
	\begin{minipage}[c][1\width]{
	   0.5\textwidth}
        \hspace{-30pt}
	   \centering
	   \includegraphics[width=0.92\textwidth]{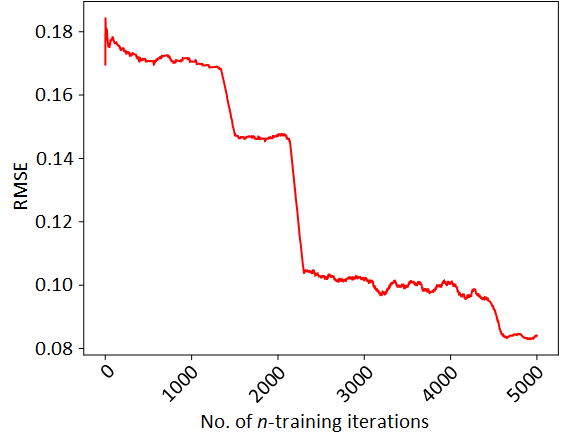}
    \vspace{-50pt}
	\end{minipage}}
 \vspace{-10pt}
\caption{{(a) $n$-updates and (b) running RMSE values w.r.t. number of iterations, with initial $n=2$ and $\alpha=0.4$ on GW.}}
\label{n8_pt4_app2}
\end{figure}

\begin{figure}[htb]
\centering
\vspace{-45pt}
  \subfloat[]{
	\begin{minipage}[c][1\width]{
	   0.5\textwidth}
    \hspace{-20pt}
	   \centering
        \includegraphics[width=0.92\textwidth]{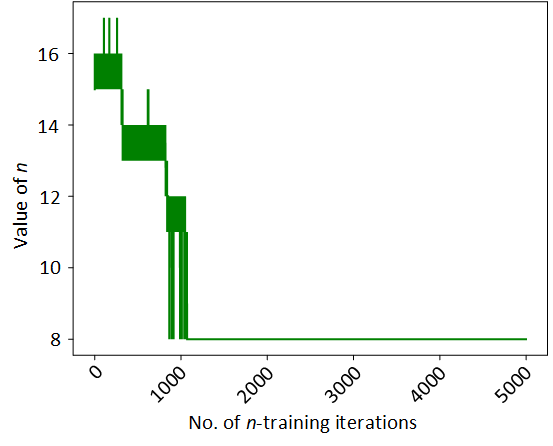}
        \vspace{-50pt}
	\end{minipage}}	
  \subfloat[]{
	\begin{minipage}[c][1\width]{
	   0.5\textwidth}
        \hspace{-30pt}
	   \centering
	   \includegraphics[width=0.92\textwidth]{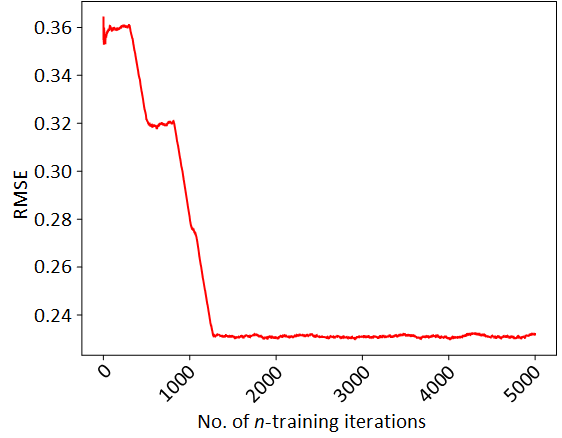}
    \vspace{-50pt}
	\end{minipage}}
 \vspace{-10pt}
\caption{{(a) $n$-updates and (b) running RMSE values w.r.t. number of iterations, with initial $n=16$ and $\alpha=0.2$ on GW.}}
\label{n16_pt2_app2}
\end{figure}
 \begin{figure}[htb]
\centering
\vspace{-45pt}
  \subfloat[]{
	\begin{minipage}[c][1\width]{
	   0.5\textwidth}
    \hspace{-20pt}
	   \centering
        \includegraphics[width=0.92\textwidth]{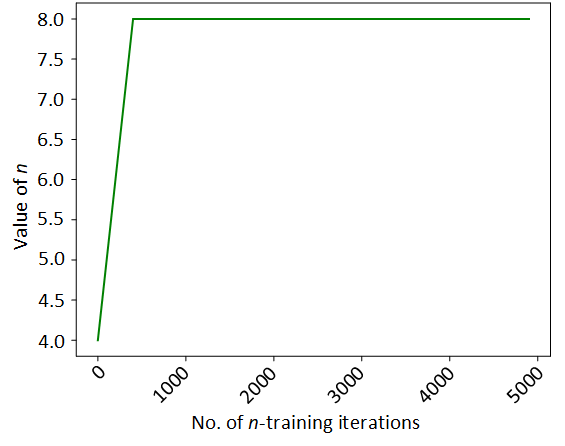}
        \vspace{-50pt}
	\end{minipage}}	
  \subfloat[]{
	\begin{minipage}[c][1\width]{
	   0.5\textwidth}
        \hspace{-30pt}
	   \centering
	   \includegraphics[width=0.92\textwidth]{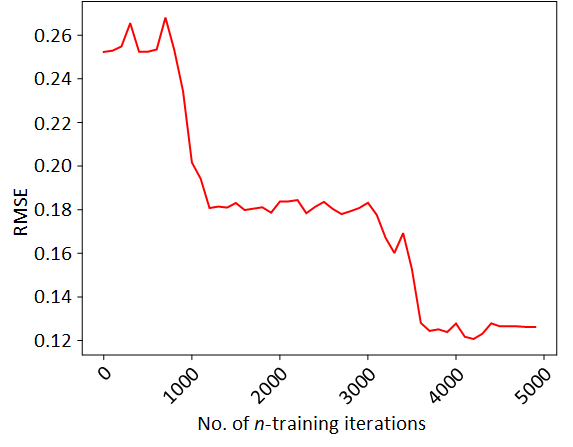}
    \vspace{-50pt}
	\end{minipage}}
 \vspace{-3pt}
\caption{(a)$n$-updates and (b)running RMSE values w.r.t. number of iterations, with initial $n=4$ and $\alpha=0.2$ on GW.}
\label{n4_pt2_app2}
\end{figure}

For each experiment, we fix a value of $\alpha$ to use in the TD($n$) algorithm and run the SDPSA algorithm to find the optimal value of $n$ for an arbitrary initial value of the same. Results of nine experiments (labeled $\#1$ to $\#9$) on two different benchmark RL environments (RW and GW) are shown here.  
In the first three experiments ($\#1$ to $\#3$), 
$\alpha=0.6$, and the initial value of $n$ is set to $16, 8$ and $4$, respectively. In the next three experiments (i.e., $\#4$ to $\#6$), $\alpha=0.4$ and the initial value of $n=16, 8, 2$, respectively. Finally, for the last three experiments (i.e., $\#7$ to $\#9$), $\alpha=0.2$ and the initial value of $n=16,4,2$, respectively. 

In each experiment, the update of $n$ is carried out 20,000/5,000 times on RW/GW (using Algorithm \ref{SPSA}), and in between any two successive updates of $n$, Algorithm \ref{n_step_TD} is run for ten episodes in order to estimate the root mean-squared error (RMSE). The RMSE for all the plots is obtained in this manner. The parameters of SDPSA, i.e., $\delta$ and the step-size $\nu$ of the slower timescale recursion (that is used to update $n$) are tuned in each experiment to achieve faster convergence. Depending on the experiment, we set the value of $\delta$ between 0.06 and 0.6, and the initial value of $\nu$ is set between 0.05 and 0.2, and it is slowly diminished after a certain number of iterations so that it finally converges to a small positive constant that continues to be below $\alpha$. Note here again that the value of $\nu$ is chosen to be smaller than $\alpha$ in each experiment to retain the two-timescale effect.

\begin{table}[t]
\centering
\renewcommand{\arraystretch}{2}
\caption{Results: Converged $n$ and RMSE on RW and GW.}
\label{tbl:result_RMS}
\scalebox{0.75}{
\begin{tabular}{|c|c|c|c|c|c|}
\cline{1-6}
\textbf{Exp.}  & \textbf{Initial $n$} & \textbf{$\alpha$} & \textbf{Converged $n$} & \textbf{Minimal RMSE} & \textbf{Converged RMSE } \\ & {RW / GW} &  {RW / GW} &  {RW / GW}  & in \cite{RL_Book} RW &  \textbf{{(SDPSA) RW / GW}} \\ \cline{1-6}
 \textbf{$\#1$} & 16 & 0.6  & \textbf{2} &  0.27 &  \textbf{{0.12 / 0.05}}    \\ \cline{1-6}
\textbf{$\#2$} & 8 &  0.6 & \textbf{2} &  0.27  &  \textbf{{0.10 / 0.04}}  \\ \cline{1-6}
\textbf{$\#3$} & 4 &  0.6 & \textbf{2}  &  0.27 &   \textbf{{0.10 / 0.04}}   \\ \cline{1-6}
\textbf{$\#4$} & 16 & 0.4  & \textbf{4} &  0.26 &  \textbf{{0.17 / 0.10}}    \\ \cline{1-6}
\textbf{$\#5$} & 8 &  0.4 & \textbf{4}  &  0.26 &  \textbf{{0.17 / 0.10}}   \\ \cline{1-6}
\textbf{$\#6$} & 2 &  0.4 & \textbf{4} &  0.26  &  \textbf{{0.17 / 0.08}}  \\ \cline{1-6}
\textbf{$\#7$} & 16 & 0.2  & \textbf{8} &  0.28 &  \textbf{{0.25 / 0.24}}    \\ \cline{1-6}
\textbf{$\#8$} & 4 &  0.2 & \textbf{8} & 0.28  &  \textbf{{0.25 / 0.12}}  \\ \cline{1-6}
\textbf{$\#9$} & 2 &  0.2 & \textbf{8}  &  0.28 &  \textbf{{0.24 / 0.12}}  \\ \cline{1-6}
\end{tabular}}
\end{table}
\begin{figure}[htb]
    \centering
    \includegraphics[width=4 in]{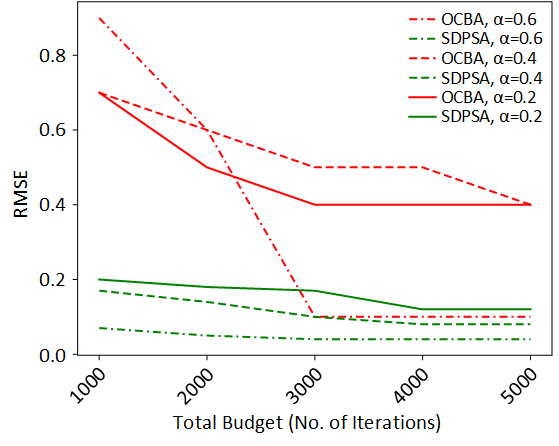}
    \caption{RMSE values w.r.t. computation budget of OCBA and SDPSA with fixed $\alpha$ on GW.}
    \label{fig:OCBA_SDPSA}
\end{figure}
\begin{table}[t]
\centering
\renewcommand{\arraystretch}{2}
\caption{Comparison Results of OCBA and SDPSA on GW.}
\label{tbl:OCBA_SDPSA}
\scalebox{1}{
\begin{tabular}{|c|c|c|c|c|}
\hline
\multirow{3}{*}{{$\alpha$ in $n$-step TD}} & \multicolumn{2}{c|}{{Time (Sec.)}} & \multicolumn{2}{c|}{{RMSE}} \\ \cline{2-5} 
                          & {OCBA}    & {SDPSA}    & {OCBA}    & {SDPSA}    \\ \hline
{0.6}                     & {588} & \textbf{{452}} & {0.10} & \textbf{{0.04}} \\ \hline
{0.4}                     & {610} & \textbf{{405}} & {0.40}  & \textbf{{0.08}} \\ \hline
{0.2}                    & {571} & \textbf{{490}} & {0.40}  & \textbf{{0.12}} \\ \hline
\end{tabular}}
\end{table}

The objective of our experiments using the SDPSA algorithm is to reach the optimal point during each run of the algorithm. All experimental results obtained from experiments $\#1$ to $\#9$ on RW and GW environments are summarized in Table \ref{tbl:result_RMS}.  Further, results of experiments $\#1$ to $\#3$ on RW are depicted graphically in plots (see Figs. \ref{n16_pt6} to \ref{n4_pt6}), where each figure has two sub-plots (a) and (b).
Here sub-plot (a) in each figure is a plot of the projected value of $n$ as a function of number of training iterations while sub-plot (b) is a plot of the running RMSE associated with the $n$-update again as a function of the number of training iterations.
From Figs. \ref{n16_pt6}(a), \ref{n8_pt6}(a), and \ref{n4_pt6}(a), it can be seen that despite having considerably different initial values as 16, 8, and 4, respectively, the $n$-update obtained from our algorithm converges to 2 in a finite number of iterations in all cases. The plots in Figs. \ref{n16_pt6}(b), \ref{n8_pt6}(b), and \ref{n4_pt6}(b) show that the RMSE decreases and converges with the number of updates of $n$ obtained from our algorithm (see Table \ref{tbl:result_RMS}). 

In Table \ref{tbl:result_RMS}, the last column and the second last column present the values of the converged RMSE using our proposed algorithm and the minimal RMSE using the existing algorithm in \cite{RL_Book}, respectively, which shows our algorithm achieves better performance in terms of RMSE. Further, the plot in Fig. \ref{n16_pt4}(a) shows that the initial value of $n$ is 16 (the value of $\alpha$ is 0.4) and the value of $n$ converges to 4, after a finite number of iterations of the $n$-update. Similar results follow for other initial values of $n$ and $\alpha=0.4$ (see Table \ref{tbl:result_RMS}). Fig. \ref{n16_pt4}(b) shows that the RMSE decreases as the number of $n$ updates increases and eventually converges. Similar results are seen for other initial values of $n$ and $\alpha=0.4$ (see Table \ref{tbl:result_RMS}).

From Figs. \ref{n16_pt2}(a), \ref{n4_pt2}(a), and \ref{n2_pt2}(a), it can be seen that initial values of $n$ are 16, 4, and 2, respectively, but the value of $n$ converges to 8 in a finite number of iterations in all cases. The plots in Figs. \ref{n16_pt2}(b), \ref{n4_pt2}(b), and \ref{n2_pt2}(b) show that the RMSE decreases and converges with the number of updates of $n$ obtained from our algorithm (see Table \ref{tbl:result_RMS}). {We present results obtained from the GW environment of experiments $\#1$, $\#2$, $\#4$, $\#6$, $\#7$ and $\#8$ (as representative experiments) in Figs. \ref{n16_pt6_app2} to \ref{n4_pt2_app2}, where parts (a) and (b) correspond to similar plots as for the RW environment (mentioned above). The experiments on GW achieve similar results as in RW and confirm the convergence and stability of our proposed algorithm across different environments.  Note that, as projected values of $n$, see \eqref{eq:projection}, are needed for the $n$-step-TD algorithm (i.e., Algorithm \ref{n_step_TD}), we report the same as the value of $n$ in the $4^{\text{th}}$ column of Table \ref{tbl:result_RMS} and all sub-plots denoted by $(a)$.}

{Table \ref{tbl:result_RMS} provides the numerical values of the performance indices obtained using the SDPSA algorithm on two different RL benchmark tasks.} These results confirm that our proposed SDPSA algorithm is effective in achieving optimal parameter values fast, starting from any arbitrary initial value of the same. {Moreover, our experimental results demonstrate the stability of the proposed algorithm on different RL benchmark environments.}

{In our final set of experiments, we present empirical comparisons with the Optimal Computing Budget Allocation (OCBA) algorithm \cite{chenlee} on the GW setting. As mentioned earlier, OCBA is a ranking and selection procedure that is widely recognised as the best algorithm in the case of discrete parameter spaces of moderate size as it judiciously allocates the simulation budget to the various parameters. Figure \ref{fig:OCBA_SDPSA} and Table \ref{tbl:OCBA_SDPSA} present comparison results of OCBA with our algorithm SDPSA. Here, we arbitrarily select experiments $\#3$, $\#6$, and $\#9$ as our representative experiments for comparison, for which we also implement the OCBA algorithm. In Figure~\ref{fig:OCBA_SDPSA}, we show plots of the RMSE for SDPSA and OCBA as a function of the total simulation budget for various values of the step-size parameter $\alpha$ of the TD($n$) scheme. Table~\ref{tbl:OCBA_SDPSA} shows the values of RMSE and total computational time taken by the two algorithms at the end of the 5,000 iterations.
Both Figure \ref{fig:OCBA_SDPSA} and Table \ref{tbl:OCBA_SDPSA} demonstrate that our proposed algorithm outperforms the existing OCBA algorithm in terms of execution time and RMSE; the proposed algorithm takes lower execution time and achieves lower error in all the considered experiments.}
 
\section{Conclusions}
\label{conclusion}

 In this paper, we considered the problem of finding the optimal value of $n$ in $n$-step TD for given choice of the learning rate parameters and proposed a discrete optimization procedure for the same. We presented a one-simulation deterministic perturbation SPSA algorithm for the case of discrete parameter optimization under noisy data. 
Our algorithm incorporates two timescales where along the faster scale, the regular $n$-step TD recursion is run for a given value of $n$ that in-turn gets updated on a slower timescale. We presented an analysis of the convergence of our proposed SDPSA algorithm using a differential inclusions-based approach and showed that the algorithm converges asymptotically to the optimal $n$ that minimizes the estimator variance. Our experiments effectively demonstrate the efficacy of our proposed algorithm and show that it beats the state-of-the-art OCBA algorithm, for discrete parameter stochastic  optimization, on benchmark RL tasks.

As future work, for the case of constant step-sizes $\alpha$ and $\nu$, one may design an algorithm for the hybrid problem of finding the pair of optimal such step-sizes $\alpha$ and $\nu$ as well as the parameter $n$, the combination of which would give the best performance. Note however that our convergence analysis has been shown for the case of diminishing step-sizes $\alpha$ and $\nu$. The same will have to be carried out for the case of constant step-sizes $\alpha$ and $\nu$ respectively. 
Finally, it would be of interest to obtain concentration bounds for set-valued maps and obtain finite-time bounds for our algorithm.





\bibliographystyle{IEEEtran}
\bibliography{n_stepTD}

\end{document}